\documentclass{article}

\usepackage[preprint]{neurips_2025}

\usepackage[utf8]{inputenc} 
\usepackage[T1]{fontenc}    
\usepackage{hyperref}       
\usepackage{url}            
\usepackage{booktabs}       
\usepackage{amsfonts}       
\usepackage{nicefrac}       
\usepackage{microtype}      
\usepackage{xcolor}         

\usepackage{subcaption}
\usepackage{graphicx}
\usepackage{amsmath}
\usepackage{multirow} 
\usepackage{amsthm}
\usepackage{hyperref}    
\usepackage{cleveref}    
\usepackage{adjustbox}
\usepackage{enumitem}
\usepackage{algorithm}
\usepackage{algpseudocode}
\usepackage{svg}
\usepackage{multirow}  
\usepackage{makecell}

\title{Policy Disruption in Reinforcement Learning:
Adversarial Attack with Large Language Models and Critical State Identification}

\author{%
  Junyong Jiang, Buwei Tian, Chenxing Xu, Songze Li, Lu Dong \\
  School of Cyber Science and Engineering\\
  Southeast University\\
  Nanjing, CN, 210000\\
  \texttt{junyongjiang@seu.edu.cn} \\
}

\newtheorem{theorem}{Theorem}

\begin{document}

\maketitle

\begin{abstract}
  Reinforcement learning (RL) has achieved remarkable success in fields like robotics and autonomous driving, but adversarial attacks—designed to mislead RL systems—remain challenging. Existing approaches often rely on modifying the environment or policy, limiting their practicality.
This paper proposes an adversarial attack method in which existing agents in the environment guide the target policy to output suboptimal actions without altering the environment. We propose a reward iteration optimization framework that leverages large language models (LLMs) to generate adversarial rewards explicitly tailored to the vulnerabilities of the target agent, thereby enhancing the effectiveness of inducing the target agent toward suboptimal decision-making. Additionally, a critical state identification algorithm is designed to pinpoint the target agent's most vulnerable states, where suboptimal behavior from the victim leads to significant degradation in overall performance. Experimental results in diverse environments demonstrate the superiority of our method over existing approaches. Our code is available at \url{https://anonymous.4open.science/r/ARCS_NIPS-B4F1}.
\end{abstract}

\section{Introduction}
RL driven by advances in deep learning, has become a key technology for sequential decision-making, achieving superhuman performance in a variety of domains including autonomous driving~\cite{liu2024augmenting, li2024think2drive}, robotic control~\cite{haarnoja2024learning, chen2024rlingua}, adversarial games~\cite{lechner2024gigastep,10662897}, and board games like Go~\cite{silver2018general}, as well as complex multi-agent collaboration~\cite{oroojlooy2023review, xiao2024macns}.
 However, despite its success in many applications, RL still faces significant challenges in robustness and security. Studies have shown that even well-trained RL policies are highly susceptible to adversarial attacks \cite{ilahi2021challenges, extra}, leading to catastrophic consequences in safety-critical domains (e.g., autonomous driving and industrial control) \cite{11004053}, thus limiting RL’s real-world deployment.

However, existing adversarial methods—whether by poisoning the environment \cite{10.5555/3524938.3525677}, perturbing observations~\cite{state3}, or injecting malicious actions~\cite{action1}—assume unfettered access to the training environment or the agent’s policy interfaces. In practice, this rarely applies, limiting their feasibility in settings where attackers cannot directly access servers or manipulate environmental states, such as commercial gaming platforms or autonomous driving systems. To circumvent this limitation, recent work has investigated adversarial methods that avoid direct environment tampering by embedding adversarial agents into the system and leveraging multi-agent interactions to indirectly disrupt policy learning~\cite{gleave2020adversarial, wu2021adversarial, guo2021adversarial}. While such methods eliminate the need for direct environment manipulation, they are typically constrained by fixed and generic attack objectives that lack task-specific guidance and fail to exploit the unique vulnerabilities of the victim policy. This limits both the robustness and generalization of the resulting attacks, particularly in dynamic or complex environments.

To address these shortcomings, we propose ARCS (Adversarial Rewards and Critical State Identification), an adaptive adversarial framework consisting of a reward iteration optimization module and a critical state identification mechanism. Specifically, we design a reward iteration optimization framework that leverages LLMs to adaptively generate customized adversarial reward functions aligned with the vulnerability of the victim policy.
Furthermore, we develop a critical state identification mechanism that selects critical states where suboptimal actions by the victim have a disproportionately large impact on task outcomes. To better exploit these states, we inject additional rewards during attacker training, guiding the policy to focus on situations where influencing the victim’s decisions yields greater adversarial returns.
Extensive experiments across diverse environments demonstrate that ARCS significantly outperforms existing adversarial policy training methods in terms of attack success rates, validating its effectiveness in adversarial policy training.
Our main contributions are summarized as follows: \begin{itemize}[leftmargin=1.5em]
\item We propose a reward iteration optimization framework that leverages LLMs to generate adversarial reward functions, enabling adaptive and targeted guidance for adversarial policy training;
\item We develop a critical state identification mechanism that selects critical states where the victim's suboptimal actions significantly affect task outcomes. These states are used to guide attacker training, enabling more effective exploitation of strategic weaknesses.

\item We introduce ARCS, a novel adversarial attack framework where existing agent guide the victim policy toward suboptimal behaviors, and validate its superiority through extensive experiments across multiple environments.
\end{itemize}

\section{Related Work}

Adversarial attacks in RL have garnered substantial attention, with a variety of approaches developed to undermine the learning and decision-making processes of RL agents \cite{10263803, standen2025adversarial}. Existing methods can be broadly categorized into environment poisoning, state perturbation, adversarial action insertion, and indirect adversarial policy training through agent interactions.

\textbf{Environment Poisoning.}
Environment poisoning attacks manipulate rewards or transition dynamics to mislead learning. Prior work has formulated optimal poisoning strategies under full environment knowledge~\cite{10.5555/3524938.3525677,10.5555/3524938.3525979}, or designed adaptive reward perturbations based on internal Q-values~\cite{10.5555/3463952.3464113}. Other methods target federated reinforcement learning by manipulating critic updates~\cite{ma2024reward}, or propose joint reward-action attacks in multi-agent systems with access to feedback channels~\cite{liu2023efficient}. These approaches typically assume privileged access to environment dynamics, internal models, or communication channels. In contrast, our method performs black-box attacks without modifying the environment or relying on internal signals, by learning tailored adversarial rewards through interactive optimization.

\textbf{State Perturbation.}
State perturbation attacks mislead agents by injecting small but adversarially crafted noise into observations~\cite{state1,state2}. Some methods optimize per-step perturbations via policy gradients~\cite{state3}, or design universal perturbations applicable across episodes~\cite{state4}. Recent work reformulates the attack in function space and employs deceptive trajectories to minimize long-term reward~\cite{qiaoben2024understanding}, but requires either policy access or surrogate models. These attacks, while effective, assume the ability to intercept and modify input streams before action selection, which is impractical in many real-world applications. However, our method avoids direct observation tampering and instead influences agent behavior solely via strategic interaction.

\textbf{Adversarial Action Insertion.}
Action-space attacks aim to mislead the agent by directly altering its selected actions \cite{bai2025rat}. Some methods inject gradient-based perturbations to shift actions within bounded budgets~\cite{action1}, while others train adversarial agents to override or replace actions through learned strategies~\cite{9147846}. Recent work proposes a decoupled adversarial policy that separately decides when and how to intervene, using a pre-built perturbation database to induce targeted actions~\cite{action2}. While effective, these approaches typically assume white-box access to the policy network or its gradients, and require control over the agent’s action channel. Our method operates purely through black-box interaction, without manipulating action outputs or requiring internal access.

\textbf{Adversarial Policy Training via Agent Interactions.}
Recent efforts have explored indirect adversarial attacks by training agents that interact with and disrupt victim policies in multi-agent environments. Early approaches adopt simple objectives, such as maximizing the adversary’s own win rate~\cite{gleave2020adversarial}, but fail to exploit specific weaknesses in the victim’s behavior. Later methods introduce crafted loss functions to amplify policy differences or degrade victim returns~\cite{wu2021adversarial,guo2021adversarial}, yet these objectives are manually designed and shared across tasks. More recently,~\cite{liu2024rethinking} focuses on attack stealth by limiting behavioral deviation, aiming to avoid detection rather than improve attack effectiveness. In contrast, our ARCS framework comprises two key components: a large language model that generates victim-specific reward functions, and a critical-state module that selectively identifies vulnerable decision points. Together, they enable adaptive and precise black-box attacks without access to the victim’s policy model or environment internals.


\section{Proposed Technique}
\subsection{Problem Scope and Assumption}

We consider a two-player adversarial setting modeled as a Markov Decision Process (MDP). One agent, referred to as the \emph{victim} (denoted $\mathcal{O}$), follows a fixed policy $\pi_{\mathcal{O}}$ and aims to accomplish a primary task. The other agent, referred to as the \emph{attacker} (denoted $\mathcal{A}$), learns an adversarial policy $\pi_{\mathcal{A}}$ to disrupt the victim's performance.

Formally, at each timestep $t$, the environment is in a state $s_t$. The victim and attacker independently observe their respective observations $o^v_t$ and $o^a_t$, and simultaneously select actions $a^v_t \sim \pi_{\mathcal{O}}(\cdot|o^v_t)$ and $a^a_t \sim \pi_{\mathcal{A}}(\cdot|o^a_t)$. These actions are executed, leading the environment to transition to the next state $s_{t+1}$ according to an unknown transition function $P(s_{t+1}|s_t, a^v_t, a^a_t)$. Both agents subsequently receive new observations and rewards based on the updated environment state.

We define the adversarial setting through the following components:
\begin{itemize}[leftmargin=1.5em]
    \item \textbf{Attacker's Goal:} To degrade the performance of a fixed victim policy $\pi_O$ by inducing suboptimal actions through interactive influence.
    \item \textbf{Attacker's Knowledge:} The attacker has no access to the victim’s architecture, parameters, gradients, or environment dynamics.
    \item \textbf{Attacker's Capability:} The attacker observes both its own and the victim’s observations, and influences the victim solely through its own actions during interaction. It cannot modify the victim’s observations, actions, rewards, or internal mechanisms.
\end{itemize}

Prior work \cite{gleave2020adversarial, guo2021adversarial, wu2021adversarial} has explored adversarial learning through agent interactions. However, these approaches often assume static adversarial objectives or partial access to victim behaviors, limiting their adaptability across dynamic environments.
In contrast, our method introduces an adaptive adversarial framework that dynamically designs reward functions and strategically identifies {critical decision points}—states where suboptimal actions by the victim have a disproportionately large impact on task outcomes—thereby enabling effective adversarial influence without requiring internal access to the victim or the environment.

\subsection{Challenge and Technical Overview}

Adversarial policy training in black-box multi-agent environments presents two core challenges: 
the lack of reward functions tailored to the specific vulnerabilities of victim agents, and 
 the difficulty of identifying pivotal states where suboptimal actions by the victim have a disproportionately large impact on task outcomes. 
Conventional approaches often employ hand-crafted or fixed reward structures that fail to generalize across tasks, leading to inefficient learning and weak adversarial impact.



\begin{figure}[ht]
\vskip 0.2in
\begin{center}
\includegraphics[width=1.0\textwidth]{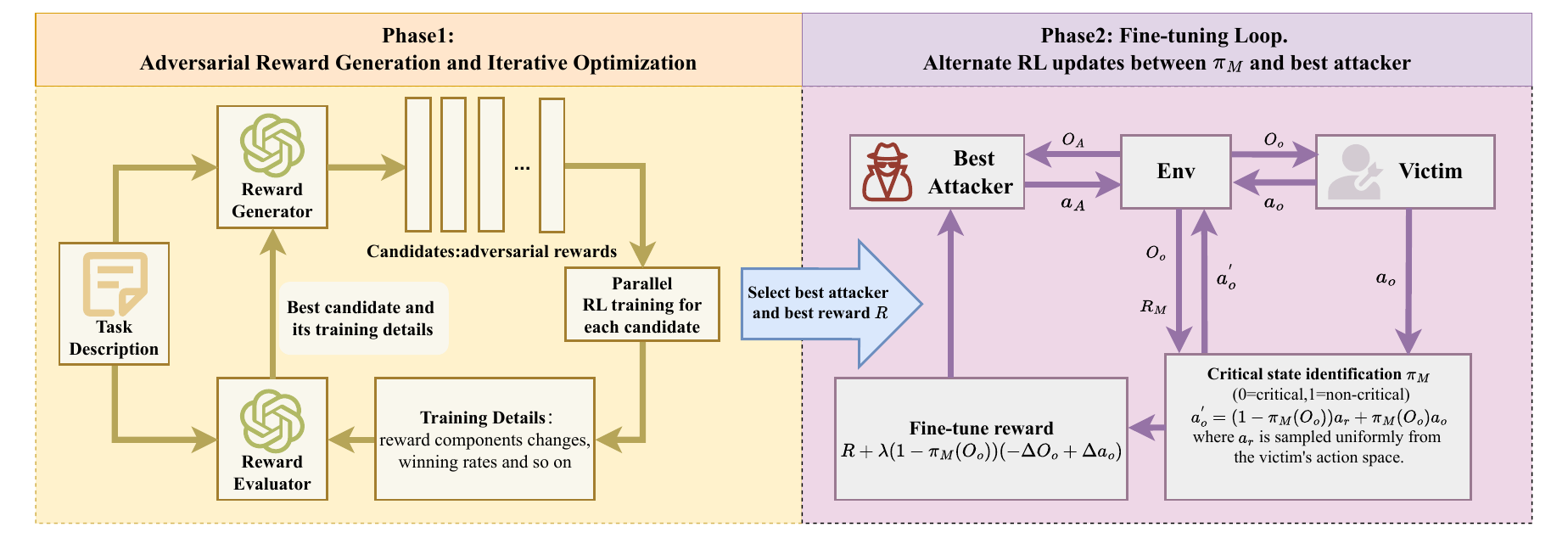}
\caption{Overview of the ARCS training framework. The left part shows LLM-guided adversarial reward generation. The right part depicts critical state identification and fine-tuning, which guide the attacker to focus perturbations on high-impact decision points.}
\label{fig:overview}
\end{center}
\vskip -0.2in
\end{figure}

\noindent
To address these challenges, we propose \textbf{ARCS}, an adaptive adversarial framework that combines LLM-guided reward generation with critical state-aware fine-tuning.
Figure~\ref{fig:overview} illustrates the ARCS framework, which integrates two modules. The left part shows the LLM-guided reward optimization process: after each round of attacker training, performance statistics are evaluated by the Reward Evaluator, which summarizes the effectiveness of each reward candidate. This feedback is then provided to the Reward Generator, which uses it to produce updated reward functions for the next training iteration.
The right part depicts the critical state identification module, where an auxiliary policy selects high-impact states. These states are used to further train the attacker policy, focusing learning on pivotal decision regions. Together, the two modules form a closed loop in which global reward structures and local state signals are jointly optimized to enhance adversarial effectiveness in a black-box setting. 

For example, in the Sumo-Human environment \cite{bansal2018emergent}, where two agents attempt to push each other out of a circular arena while maintaining balance, existing methods primarily rely on a simple win signal without explicitly targeting the behavioral weaknesses of the victim, resulting in inefficient adversarial learning. Our LLM-generated rewards incorporate detailed elements such as penalties for tilting and bonuses for destabilization, enabling more effective and targeted training. 
Similarly, enhancing attacker learning at critical states rather than uniformly across all states results in more efficient and targeted disruption.

\subsection{Technical Details}
\label{sec:method}
\textbf{Adversarial Reward Generation and Iterative Optimization.}  
We employ two LLMs in this process: a \emph{Reward Generator}, responsible for producing adversarial reward functions, and a \emph{Reward Evaluator}, tasked with assessing the effectiveness of the generated rewards.

To support effective reward generation, we design structured prompts comprising three components: a base prompt that describes the task context, and two functional templates dedicated to reward generation and reward evaluation, respectively. The base prompt specifies environment settings, available variables, and task objectives. The generation template prompts the Reward Generator to produce new candidate rewards, while the evaluation template instructs the Reward Evaluator to select the best-performing reward from previous rounds based on training outcomes. Full prompt templates are provided in Appendix~C.

In the reward generation module, the process begins with the Reward Generator producing a set of candidate adversarial reward functions. Each candidate is used to train the attacker policy for a fixed number of steps. During training, key statistics—such as changes in individual reward components and the trajectory of success rates—are recorded as feedback for evaluation.

After training, the structured prompt templates are dynamically updated in two ways. First, the collected training details are incorporated into the prompt provided to the Reward Evaluator, enabling it to assess the effectiveness of all candidate rewards based on empirical evidence. Second, only the best-performing reward function identified by the Reward Evaluator, along with its associated training details, is incorporated into the prompt for the next invocation of the Reward Generator. This ensures that the generation of new adversarial rewards is guided by the most effective prior experience in subsequent iterations.

A key distinction between the initial and subsequent iterations lies in the availability of empirical feedback. During the first iteration, the Reward Generator is guided solely by static task information, including the environment description and available variables. In subsequent iterations, the prompts are dynamically enriched with training feedback derived from the best-performing reward function selected by the Reward Evaluator, enabling the Reward Generator to iteratively refine its outputs based on the most effective prior experience.

Through several rounds of generation, training, evaluation, and refinement, an effective adversarial reward function is ultimately obtained. Once the optimal adversarial reward function is determined, it is directly used for training the attacker agent. 

\textbf{Critical State Identification.}
After obtaining the optimal adversarial reward function $R$, we introduce a critical state identification mechanism to further enhance adversarial effectiveness. The key idea is to identify pivotal states where suboptimal actions by the victim have a strong impact on outcomes, and to guide attacker training to exploit these situations more effectively.

To achieve this, we introduce an auxiliary binary policy $\pi_{\mathcal{M}}$, where $\pi_{\mathcal{M}}(s) \in \{0,1\}$, to determine whether to replace the victim’s action at state $s$. The perturbed policy $\pi$ is defined as:
\[
\pi(s) = 
\begin{cases}
\pi_{\mathcal{O}}(s), & \text{if } \pi_{\mathcal{M}}(s) = 1, \\
\text{random action}, & \text{if } \pi_{\mathcal{M}}(s) = 0.
\end{cases}
\]
A state is labeled as critical if replacing the victim’s action with a random one causes a significant drop in overall performance.
Our objective is to minimize the victim’s cumulative return by altering its actions at only a limited number of states:
\begin{equation}
\label{eq:problem}
\begin{aligned}
\text{minimize} \quad & \eta(\pi) \\
\text{subject to} \quad & C_2 \leq N \leq C_1,
\end{aligned}
\end{equation}
where $N$ denotes the number of perturbed states, and $\eta(\pi)$ is the expected cumulative reward of the victim following the perturbed policy.

However, directly minimizing $\eta(\pi)$ is challenging due to its dependence on future state distributions. To facilitate stable and efficient updates, we adopt a local approximation around the current policy.
Specifically, we denote $\pi_{\text{old}}$ as the perturbed policy from the previous iteration. To constrain the update within a trust region and prevent large policy shifts, we approximate $\eta(\pi)$ with a first-order surrogate objective centered at $\pi_{\text{old}}$:
\begin{equation}
L_{\pi_{\text{old}}}(\pi) = \eta(\pi_{\text{old}}) + \sum_s \rho_{\pi_{\text{old}}}(s) \sum_a \pi(a|s) A_{\pi_{\text{old}}}(s,a),
\end{equation}
where $\rho_{\pi_{\text{old}}}(s)$ denotes the discounted visitation frequency under $\pi_{\text{old}}$, and $A_{\pi_{\text{old}}}(s,a)$ is the corresponding advantage function.
To further restrict the extent of policy change and prevent drastic deviations, we introduce a regularization term based on the maximum KL divergence between the new and old policies:
\begin{equation}
M(\pi) = L_{\pi_{\text{old}}}(\pi) + C \cdot \max_s \text{KL}(\pi_{\text{old}}(\cdot|s) \parallel \pi(\cdot|s)),
\end{equation}
where $C$ is a positive regularization coefficient.
We formally state the following result:

\begin{theorem}[Policy Degradation Monotonicity]
\label{theorem:monotonicity}
Minimizing $M(\pi)$ guarantees non-increasing expected return:
\[
\eta(\pi) \leq \eta(\pi_{\text{old}}).
\]
\end{theorem}

See Appendix~A.1 for the proof of Theorem~\ref{theorem:monotonicity}. Thus, the critical state identification problem is reduced to minimizing $M(\pi)$ under constraints on the number of perturbed states.


To efficiently handle the constraint $C_2 \leq N \leq C_1$, we reformulate the original constrained problem into an unconstrained dual form. Specifically, we first define the clipped surrogate objective:
\[
f(\pi) = \mathbb{E}_t \left[ \min \left( r_t A_t, \text{clip}(r_t, 1-\epsilon, 1+\epsilon) A_t \right) \right],
\]
where $r_t$ denotes the likelihood ratio between the new and old policies, and $A_t$ is the estimated advantage function.

\begin{theorem}[Optimization Reformulation]
\label{theorem:unconstrained}
The critical state identification problem can be reformulated as the following unconstrained optimization:
\[
\max_{\nu_1, \nu_2} \min_{\pi} f(\pi) + g_1(\nu_1) + g_2(\nu_2),
\]
where the penalty terms are defined as
\[
g_1(\nu_1) = \frac{\left[\max\left(\nu_1 + d_1(N - C_1), 0\right)\right]^2 - \nu_1^2}{2d_1}, \quad
g_2(\nu_2) = \frac{\left[\max\left(\nu_2 + d_2(C_2 - N), 0\right)\right]^2 - \nu_2^2}{2d_2},
\]
and $d_1$, $d_2 > 0$ are positive penalty coefficients.
\end{theorem}

This reformulation converts the original constrained optimization in Equation~(\ref{eq:problem}) into a dual form, where the upper and lower bounds on the number of perturbed states are softly enforced via quadratic penalties. This structure makes the problem amenable to standard unconstrained optimization techniques.
The proof of Theorem~\ref{theorem:unconstrained} is provided in Appendix~A.2.

The attacker optimizes the auxiliary policy $\pi_{\mathcal{M}}$ and dual variables $\nu_1, \nu_2$ jointly using proximal policy optimization(PPO) \cite{ppo} and gradient ascent.
\begin{algorithm}[h]
\caption{Training Algorithm of ARCS Framework}
\label{alg:framework}
\begin{algorithmic}[1]
\Require Victim policy $\pi_{\mathcal{O}}$, reward optimization rounds $N_{\text{reward}}$, candidate reward count $N_{\text{cand}}$, critical state update interval $K$
\Ensure Final adversarial attacker policy $\pi_{\mathcal{A}}$

\State Initialize structured prompt templates with static task information

\For{iteration $=1$ to $N_{\text{reward}}$}
    \State Generate $N_{\text{cand}}$ candidate rewards using the Reward Generator
    \ForAll{candidate rewards \textbf{in parallel}}
        \State Train attacker policy $\pi_{\mathcal{A}}$ for a few steps under each candidate reward
        \State Record training details (reward component changes, success rate trajectories)
    \EndFor
    \State Provide all candidates and their training details to the Reward Evaluator
    \State Select the best-performing reward based on evaluation
    \State Update the prompt for the Reward Generator using the best reward's training feedback
\EndFor

\State Finalize the selected adversarial reward $R$

\State Pre-train attacker policy $\pi_{\mathcal{A}}$ using $R$ via PPO 
\State Simultaneously train transition model $\tilde{P}$ and victim policy estimator $\tilde{\pi}_{\mathcal{O}}$ using supervised losses Equation(\ref{eq:Ppi})

\For{update $=1$ to fine-tuning steps}
    \If{update mod $K$ == 0}
        \State Re-optimize critical state identification policy $\pi_{\mathcal{M}}$ by Theorem \ref{theorem:unconstrained}
    \EndIf
    \State Update attacker policy $\pi_{\mathcal{A}}$ using PPO with $R_{\text{total}}$ 
    \State Simultaneously update $\tilde{P}$ and $\tilde{\pi}_{\mathcal{O}}$ using supervised losses Equation(\ref{eq:Ppi})

\EndFor

\end{algorithmic}
\end{algorithm}

\textbf{Policy Fine-tuning.}
After pre-training the attacker policy $\pi_{\mathcal{A}}$ with the selected adversarial reward $R$, we proceed to a fine-tuning stage aimed at further exploiting the vulnerabilities of the victim by focusing on critical states. To facilitate this process, we introduce two auxiliary models: a transition model $\tilde{P}$ and a victim policy estimator $\tilde{\pi}_{\mathcal{O}}$. Specifically, the transition model $\tilde{P}: O_o \times A_o \times A_a \rightarrow O_o$ predicts the victim's next observation given the current observation and both agents' actions, while the victim policy estimator $\tilde{\pi}_{\mathcal{O}}: O_o \rightarrow A_o$ predicts the victim’s next action based on its observation.
These models are trained using supervised objectives:
\begin{equation}
l_P = \sum_{t} \left\| \tilde{P}(O_o^t, A_o^t, A_a^t) - O_o^{t+1} \right\|_2, \quad
l_{\pi} = \sum_{t} \left\| \tilde{\pi}_{\mathcal{O}}(O_o^t) - A_o^t \right\|_2.
\label{eq:Ppi}
\end{equation}

To guide the attacker toward inducing impactful but minimally invasive perturbations, we define a fine-tuning reward $R_{ft}$ that encourages subtle changes in the victim’s observations while causing amplified deviations in the victim’s subsequent actions. Concretely, we introduce two deviation measures:
\begin{equation}
\Delta O_o = \left\| \tilde{P}(O_o^t, A_o^t, \pi_{\mathcal{A}}(O_a^t)) - O_o^{t+1} \right\|_2, \quad
\Delta A_o = \left\| \tilde{\pi}_{\mathcal{O}}(\tilde{P}(O_o^t, A_o^t, \pi_{\mathcal{A}}(O_a^t))) - A_o^{t+1} \right\|_2,
\end{equation}
which respectively quantify the discrepancy in the victim’s predicted future observation and action.
The fine-tuning reward is then constructed as:
\begin{equation}
R_{ft} = (-\Delta O_o + \Delta A_o) \cdot (1 - \pi_{\mathcal{M}}(O_o^{t+1})),
\end{equation}
where the negative $\Delta O_o$ term encourages the attacker to induce minimal changes in the victim’s observations, while the positive $\Delta A_o$ term promotes larger shifts in the victim’s resulting actions.  
This design incentivizes the attacker to subtly influence the shared environment in ways that cause significant behavioral deviations in the victim, thereby amplifying adversarial effectiveness.
The total reward used to update the attacker is then defined as:
\begin{equation}
R_{\text{total}} = R + \lambda R_{ft},
\end{equation}
where $\lambda$ controls the strength of the fine-tuning signal relative to the original adversarial objective. The attacker policy $\pi_{\mathcal{A}}$ is trained using the PPO algorithm with the composite reward $R_{\text{total}}$.
To maintain effective identification of critical states, the auxiliary policy $\pi_{\mathcal{M}}$ is periodically re-optimized every $K$ attacker updates by solving the constrained optimization problem in Theorem \ref{theorem:unconstrained}. The overall training procedure is summarized in Algorithm~\ref{alg:framework}.

\section{Evaluation}
\subsection{Experiment Design}
In this section, we conduct experiments in three MuJoCo environments: Sumo-Human, You-Shall-Not-Pass, and Kick-and-Defend, as well as three autonomous driving environments. A brief description of each environment is provided in Appendix~B. To evaluate the effectiveness of {ARCS}, we compare it against two representative baseline methods.
{Baseline1}\cite{gleave2020adversarial} trains adversarial agents by maximizing a sparse win/loss reward without adapting the attack objectives to different victim agents.
{Baseline2}\cite{guo2021adversarial} further enhances adversarial policy training by jointly maximizing the attacker’s reward and minimizing the victim’s reward, but still relies on a fixed surrogate objective shared across all victims.
In contrast, {ARCS} generates customized adversarial rewards tailored to each victim's specific vulnerabilities, enabling the attacker to adapt its strategy to different victim policies and achieve stronger, more targeted attacks.

To comprehensively analyze the contributions of adversarial rewards and critical state identification, we design a series of comparative experiments. We begin by evaluating an ablation variant, denoted as AR, which includes the adversarial reward optimization module but omits the critical state identification mechanism. Comparing AR with Baseline1 and Baseline2 allows us to assess the role of adaptive reward design in enhancing adversarial policy training.
Building on this setup, we apply critical state-based fine-tuning to the AR variant to form the complete ARCS framework. We then compare the performance of ARCS against its pre-finetuned version (AR), as well as the two baselines, both with and without fine-tuning. This allows us to quantify the effectiveness of the critical state identification module in improving adversarial performance when combined with reward optimization.
Finally, to directly validate the effectiveness of the critical state identification module, we perform a perturbation-based analysis across multiple environments. This involves comparing victim failure rates under no perturbation, randomly applied perturbations, and targeted perturbations restricted to identified critical states. This validation highlights the module's ability to locate states that have a disproportionately large impact on the victim’s performance.
Detailed training procedures, hyperparameter settings, and model configurations are provided in Appendix~E.




\subsection{Experiment Results}
\subsubsection{Effectiveness of Learned Adversarial Rewards}

We next provide a detailed view of how our framework generates adaptive adversarial rewards in practice. Guided by structured prompts that incorporate task descriptions, environment variables, and prior training feedback (see Appendix~C), GPT-4o is used as both the Reward Generator and the Evaluator. In each iteration, it produces four candidate reward functions, which are used to train attacker agents independently. Based on training outcomes, the Evaluator selects the most effective candidate to guide the next generation. Notably, the process converges within just four rounds, requiring only 32 API calls to obtain effective task-adaptive rewards. 

The final adversarial reward for Sumo-Human integrates dense shaping terms, sparse interaction signals, and terminal bonuses. The dense reward combines a standing score, determined by torso height, uprightness, and movement stability, with a combat score that captures ring control, opponent destabilization, and physical advantage. Their weights are adjusted dynamically according to the attacker's win rate, emphasizing stability in early training and aggression as performance improves. Additional terms include sparse rewards based on the relative advantage between agents, penalties for energy usage and long episodes, and outcome-based bonuses. This structure allows the attacker to learn progressively aggressive and effective behaviors.
Complete reward functions for all environments are provided in Appendix~D.

To evaluate the effectiveness of the generated rewards, we compare the AR against Baseline1 and Baseline2 across three MuJoCo environments and three autonomous driving environments.  The results are shown in Figure~\ref{fig:adversarial_training}.

\begin{figure}[h]
    \centering
    \begin{subfigure}[b]{0.3\textwidth}
        \includegraphics[width=\linewidth]{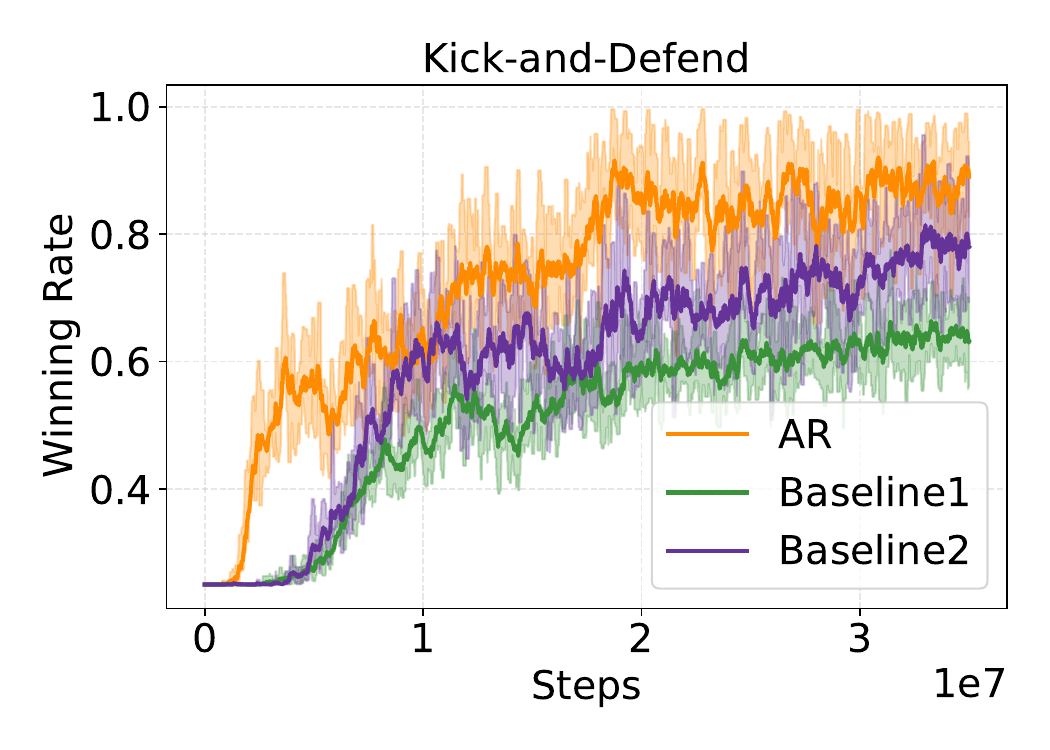}
        \caption{Kick-and-Defend}
        \label{fig:example3}
    \end{subfigure}
    \hspace{0.01\linewidth}
    \begin{subfigure}[b]{0.3\textwidth}
        \includegraphics[width=\linewidth]{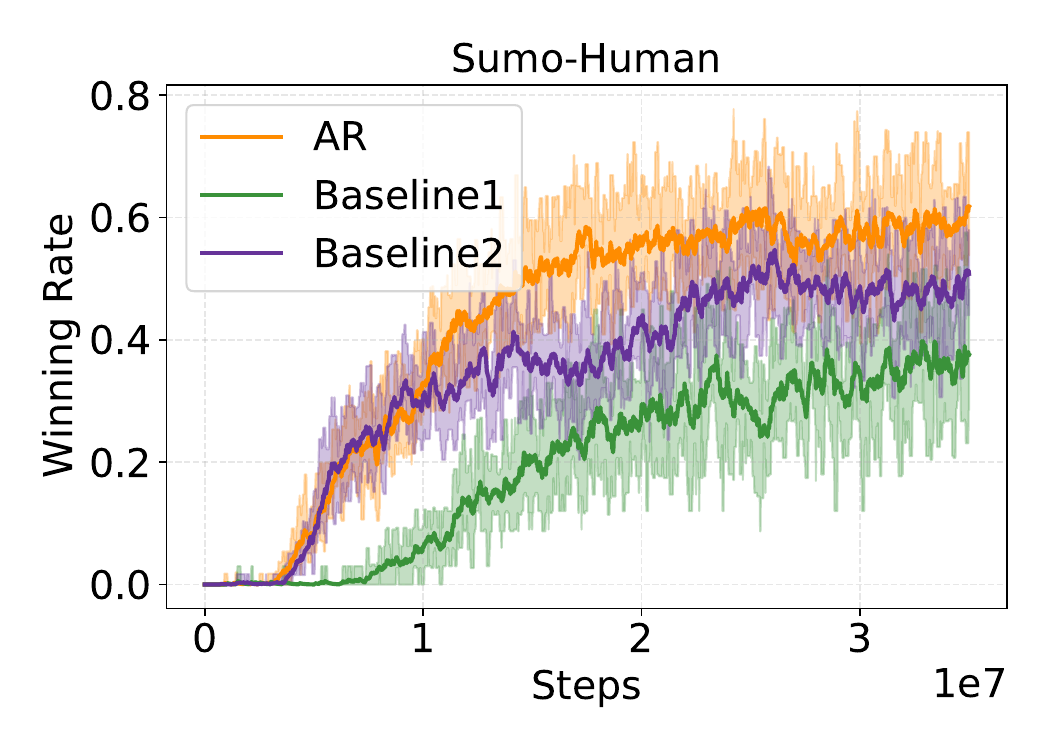}
        \caption{Sumo-Human}
        \label{fig:example4}
    \end{subfigure}
    \hspace{0.01\linewidth}
    \begin{subfigure}[b]{0.3\textwidth}
        \includegraphics[width=\linewidth]{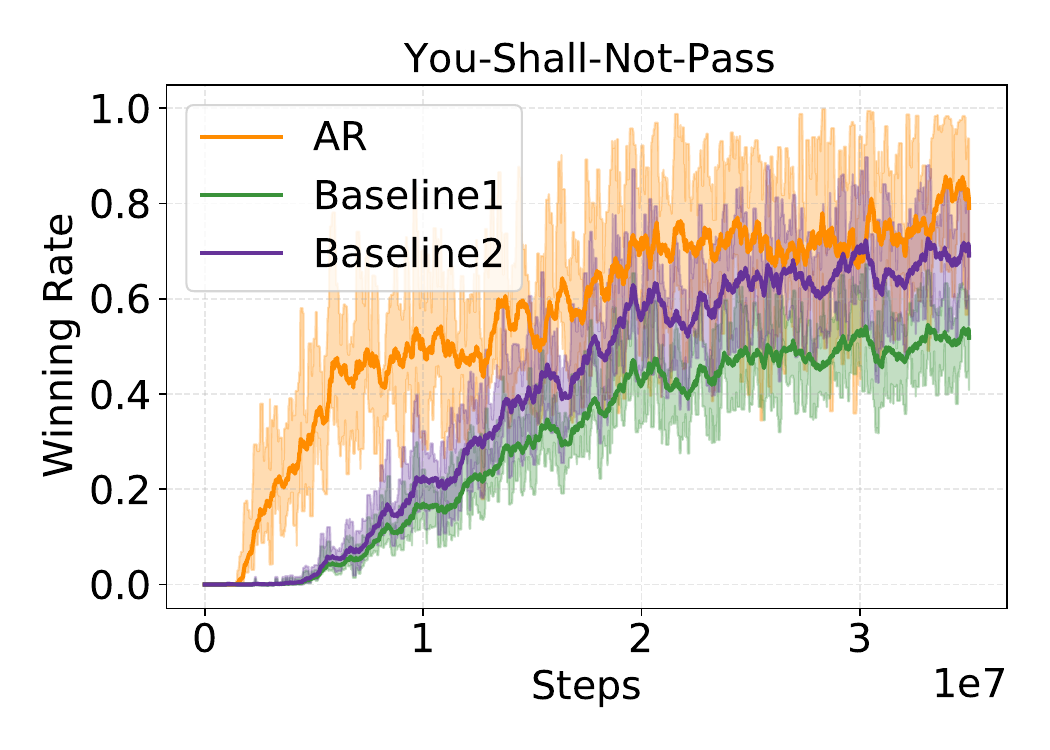}
        \caption{You-Shall-Not-Pass}
        \label{fig:example5}
    \end{subfigure}

    \vspace{0.01\linewidth}  

    \begin{subfigure}[b]{0.3\textwidth}
        \includegraphics[width=\linewidth]{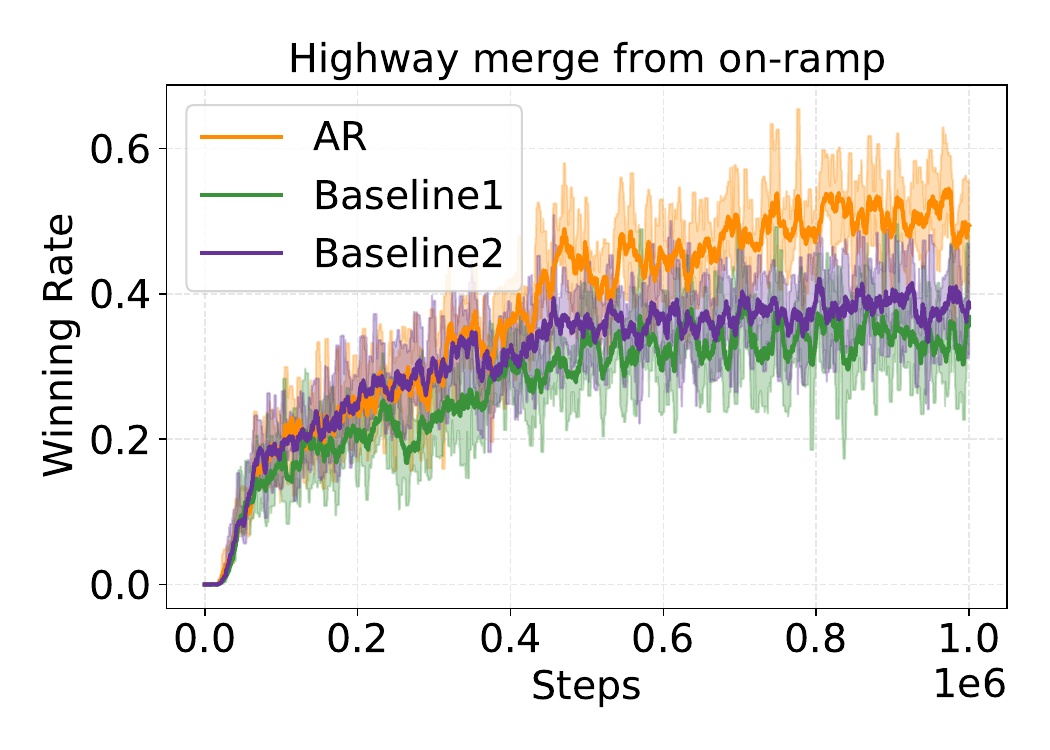}
        \caption{Highway merge from on-ramp}
        \label{fig:example6}
    \end{subfigure}
    \hspace{0.01\linewidth}
    \begin{subfigure}[b]{0.3\textwidth}
        \includegraphics[width=\linewidth]{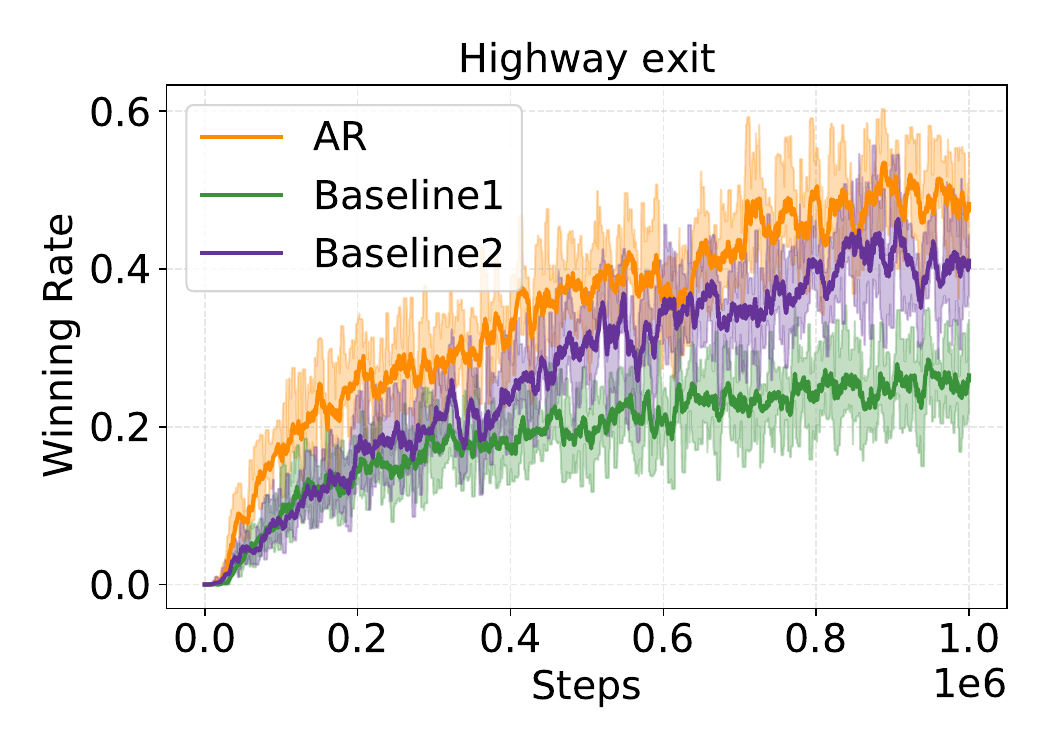}
        \caption{Highway exit}
        \label{fig:example7}
    \end{subfigure}
    \hspace{0.01\linewidth}
    \begin{subfigure}[b]{0.3\textwidth}
        \includegraphics[width=\linewidth]{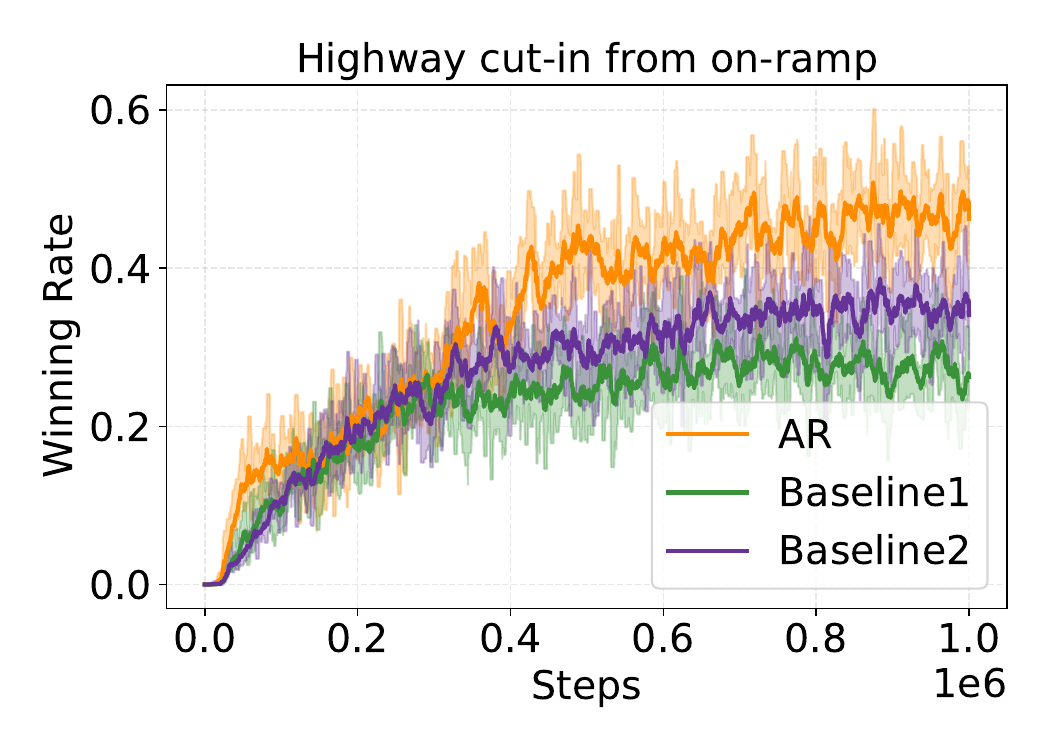}
        \caption{Highway cut-in from on-ramp}
        \label{fig:example8}
    \end{subfigure}

    \caption{Comparison of attack success rates between AR and baselines across six environments.}
    \label{fig:adversarial_training}
\end{figure}

As shown in the results, AR consistently outperforms both baselines across all environments. This improvement comes from the use of adaptive adversarial rewards that are not only aligned with the task environment, but more importantly, tailored to exploit the specific weaknesses of the victim policy. Unlike fixed or generic reward designs, our method provides targeted guidance that enables the attacker to interfere more effectively. These results demonstrate that our reward optimization framework leads to stronger and more adaptable attacks, even without access to the victim’s internal information.

\subsubsection{Effectiveness of Critical State Fine-tuning}

To evaluate the complete ARCS framework, we apply fine-tuning to the attacker policy using the identified critical states as guidance. Specifically, the attacker is first trained with the optimized adversarial reward, and then further fine-tuned with critical state, as described in Section~\ref{sec:method}. This two-stage process enables the attacker to refine its adversarial strategy by focusing on high-impact decision points. We compare attack success rates before and after fine-tuning across six environments, and analyze the results below. The hyperparameters used in reward fine-tuning and critical state selection are listed in Appendix~E.

\begin{table}[ht]
\centering
\caption{Attack success rates across six environments. Columns represent different methods with or without CS. ARCS corresponds to AR with CS.}
\label{tab:finetune_results}
\renewcommand{\arraystretch}{1.1}
\setlength{\tabcolsep}{3pt}
\begin{tabular}{ccccccc}
\toprule
\multirow{2}{*}{\textbf{Environment}} & \multicolumn{6}{c}{\textbf{Method}} \\
\cmidrule{2-7}
& Baseline1 & Baseline1CS & Baseline2 & Baseline2CS & AR & \textbf{ARCS} \\
\midrule
Sumo-Human         & 0.37 & 0.43 & 0.51 & 0.59 & 0.60 & \textbf{0.65} \\
You-Shall-Not-Pass & 0.51 & 0.57 & 0.74 & 0.79 & 0.80 & \textbf{0.85} \\
Kick-and-Defend    & 0.68 & 0.75 & 0.80 & 0.86 & 0.87 & \textbf{0.91} \\
Highway merge from on-ramp        & 0.33 & 0.37 & 0.38 & 0.44 & 0.53 & \textbf{0.60} \\
Highway exit     & 0.28 & 0.33 & 0.41 & 0.45 & 0.51 & \textbf{0.55} \\
Highway cut-in from on-ramp         & 0.27 & 0.31 & 0.33 & 0.38 & 0.51 & \textbf{0.54} \\
\bottomrule
\end{tabular}
\end{table}


As shown in Table~\ref{tab:finetune_results}, critical state fine-tuning improves the performance of all methods across all environments. This confirms that identifying and leveraging high-impact decision states is a broadly effective enhancement, regardless of the underlying reward structure. By guiding the attacker to focus on the most decisive moments, this mechanism enables more precise and efficient adaptation.

Among all methods, ARCS consistently achieves the highest post-finetuning success rates. Although its improvement margin is smaller due to a stronger starting point, this highlights the strength of our reward optimization framework. 
Combined with critical state-based training, it yields highly effective adversarial strategies. These results highlight the importance of integrating adaptive reward optimization with state-aware policy refinement in adversarial reinforcement learning.




\subsubsection{Perturbation-based Validation of Critical States}

To comprehensively validate the effectiveness of the critical state identification module, we conducted perturbation experiments across six environments. In each environment, we compared victim failure rates under three conditions: (1) no perturbation, where the victim followed its original policy throughout the episode; (2) random perturbation, where an equal number of randomly selected states were perturbed; and (3) critical state perturbation, where perturbations were applied only at states identified by our critical state selector.
In our setup, a perturbation means forcing the victim to take a random action at a specific state, thereby disrupting its original decision-making process. This enables us to test whether the states selected by our method indeed correspond to moments where the victim is most vulnerable to disruption.
\begin{table}[h]
  \centering
  \caption{Victim failure rates under different perturbation conditions across six environments. Critical state perturbation leads to consistently higher failure rates.}
  \label{tab:critical_state_full_results}
  \renewcommand{\arraystretch}{1.1}
  \setlength{\tabcolsep}{8pt}
  \begin{tabular}{cccc}
    \toprule
     \textbf{Environment} & \makecell{No \\ Perturbation} & \makecell{Random \\ Perturbation} & \makecell{\textbf{Critical State} \\ \textbf{Perturbation}} \\
    \midrule
    Sumo-Human          & 0.65 & 0.68 & \textbf{0.88} \\
    You-Shall-Not-Pass  & 0.52 & 0.57 & \textbf{0.79} \\
    Kick-and-Defend     & 0.48 & 0.51 & \textbf{0.83} \\
    Highway merge from on-ramp          & 0.43 & 0.47 & \textbf{0.72} \\
    Highway exit       & 0.39 & 0.45 & \textbf{0.70} \\
    Highway cut-in from on-ramp          & 0.41 & 0.46 & \textbf{0.68} \\
    \bottomrule
  \end{tabular}
\end{table}
As shown in Table~\ref{tab:critical_state_full_results}, perturbing the identified critical states consistently led to significantly higher victim failure rates compared to both the random perturbation and no-perturbation settings. For instance, in the \textit{Sumo-Human} environment, the failure rate increased from 0.65 (no perturbation) and 0.68 (random perturbation) to 0.88 under critical state perturbation. Similar trends were observed across all six environments, confirming that the identified critical states indeed correspond to pivotal decision-making points. These results show that our critical state identification module effectively locates high-impact decision points, where minimal intervention leads to maximal disruption.

\section{Conclusion}

In this paper, we presented ARCS, an adaptive adversarial policy training framework combining adversarial reward optimization with critical state identification. Unlike existing methods relying on static objectives or direct environment manipulation, ARCS employs large language models to adaptively generate rewards targeting specific vulnerabilities and strategically identifies critical states to enhance attack effectiveness. Experiments across multiple MuJoCo environments and autonomous driving environments validate the framework’s effectiveness, demonstrating clear improvements over baseline methods. Future work will explore extending ARCS to more complex multi-agent scenarios and integrating advanced techniques such as meta-learning and multi-task learning for enhanced robustness and adaptability.


\bibliographystyle{unsrt}  
\bibliography{reference}

\newpage
\appendix
\setcounter{theorem}{0}
\section{Proof of Theorems}
\label{appendix:proofs}

In this appendix, we provide detailed proofs for the theoretical results presented in Section~3. These proofs establish that minimizing the surrogate objective ensures the non-increasing behavior of the victim's expected cumulative reward, and demonstrate the equivalence between the constrained and unconstrained optimization formulations.

\subsection{Proof of Theorem 1}

\begin{theorem}
Minimizing $M(\pi)$ guarantees that the expected cumulative reward does not increase, that is,
\[
\eta(\pi) \leq \eta(\pi_{\text{old}}).
\]
\end{theorem}

\begin{proof}
We first introduce the following key inequality from the trust region policy optimization (TRPO) framework~\cite{trpo}:

\begin{equation}
|\eta(\pi) - L_{\pi_{\text{old}}}(\pi)| \leq C \cdot \max_s \text{KL}(\pi_{\text{old}}(\cdot|s) || \pi(\cdot|s)),
\label{eq:trpo_bound}
\end{equation}
where $C > 0$ is a constant depending on the reward scale and discount factor.

From this, we directly obtain:
\begin{equation}
\eta(\pi) \leq L_{\pi_{\text{old}}}(\pi) + C \cdot \max_s \text{KL}(\pi_{\text{old}}(\cdot|s) || \pi(\cdot|s)) = M(\pi).
\label{eq:upperbound}
\end{equation}

Meanwhile, observe that at $\pi = \pi_{\text{old}}$, we have:
\[
L_{\pi_{\text{old}}}(\pi_{\text{old}}) = \eta(\pi_{\text{old}}),
\quad
\max_s \text{KL}(\pi_{\text{old}}(\cdot|s) || \pi_{\text{old}}(\cdot|s)) = 0,
\]
thus
\[
M(\pi_{\text{old}}) = \eta(\pi_{\text{old}}).
\]

Because $M(\pi)$ is minimized over $\pi$, we have:
\begin{equation}
M(\pi) \leq M(\pi_{\text{old}}) = \eta(\pi_{\text{old}}).
\label{eq:M_bound}
\end{equation}

Combining inequalities~\eqref{eq:upperbound} and~\eqref{eq:M_bound}, we conclude:
\[
\eta(\pi) \leq M(\pi) \leq \eta(\pi_{\text{old}}),
\]
which completes the proof.
\end{proof}

\subsection{Proof of Theorem 2}

\begin{theorem}
The constrained optimization problem
\[
\min_{\pi} \eta(\pi) \quad \text{subject to} \quad C_2 \leq N \leq C_1
\]
is equivalent to solving the unconstrained optimization
\[
\max_{\nu_1, \nu_2} \min_{\pi} f(\pi) + g_1(\nu_1) + g_2(\nu_2),
\]
where
\[
f(\pi) = \mathbb{E}_t \left[ \min \left( r_t A_t, \, \text{clip}(r_t, 1-\epsilon, 1+\epsilon) A_t \right) \right],
\]
\[
g_1(\nu_1) = \frac{[\max(\nu_1 + d_1(N - C_1), 0)]^2 - \nu_1^2}{2d_1},
\quad
g_2(\nu_2) = \frac{[\max(\nu_2 + d_2(C_2 - N), 0)]^2 - \nu_2^2}{2d_2}.
\]
\end{theorem}

\begin{proof}
We start from the original objective:
\[
\min_{\pi} \eta(\pi) \quad \text{subject to} \quad C_2 \leq N \leq C_1.
\]

Directly minimizing $\eta(\pi)$ can be challenging due to instability caused by large policy updates. According to Theorem~1, minimizing the surrogate objective
\[
M(\pi) = L_{\pi_{\text{old}}}(\pi) + C \cdot \max_s \text{KL}(\pi_{\text{old}}(\cdot|s) \parallel \pi(\cdot|s))
\]
guarantees that $\eta(\pi)$ does not increase, thus providing a stable surrogate for optimization.

Therefore, the constrained problem is equivalently transformed into:
\[
\min_{\pi} M(\pi) \quad \text{subject to} \quad C_2 \leq N \leq C_1.
\]

Next, based on the TRPO theory~\cite{trpo} and PPO approximations~\cite{ppo}, the surrogate objective $M(\pi)$ can be further approximated by the clipped objective:
\[
f(\pi) = \mathbb{E}_t \left[ \min\left(r_t A_t, \, \text{clip}(r_t, 1-\epsilon, 1+\epsilon) A_t \right) \right].
\]
Thus, we rewrite the constrained optimization as:

\[
\begin{aligned}
\min_{\pi, u_1, u_2} \quad & f(\pi) \\
\text{subject to} \quad 
& N(\pi) - C_1 + u_1 = 0, \\
& C_2 - N(\pi) + u_2 = 0, \\
& u_1 \geq 0,\quad u_2 \geq 0.
\end{aligned}
\]
where $u_1$ and $u_2$ are non-negative slack variables that allow soft constraint handling.

The corresponding augmented Lagrangian is:
\[
\begin{aligned}
\mathcal{L}(\pi, u_1, u_2, \nu_1, \nu_2) =\; & f(\pi) \\
& + \nu_1 \big(N(\pi) - C_1 + u_1\big) + \frac{d_1}{2} \big(N(\pi) - C_1 + u_1\big)^2 \\
& + \nu_2 \big(C_2 - N(\pi) + u_2\big) + \frac{d_2}{2} \big(C_2 - N(\pi) + u_2\big)^2
\end{aligned}
\]
where $\nu_1, \nu_2 \geq 0$ are dual variables.

Grouping terms related to $u_1$ and $u_2$, define:
\[
P(u) = \nu (h(N(\pi)) + u) + \frac{d}{2}(h(N(\pi)) + u)^2,
\]
where $h(N(\pi))$ represents the constraint violation term ($h(N(\pi)) = N(\pi) - C_1$ for $u_1$ and $h(N(\pi)) = C_2 - N(\pi)$ for $u_2$).

Since $P(u)$ is convex in $u$, we can compute the minimum over $u \geq 0$ explicitly:
\[
\min_{u \geq 0} P(u) = 
\begin{cases}
-\frac{\nu^2}{2d}, & \text{if } -\frac{\nu}{d} - h(N(\pi)) \geq 0, \\
\frac{(\nu + d h(N(\pi)))^2 - \nu^2}{2d}, & \text{otherwise}.
\end{cases}
\]
which can be compactly written as:
\[
\min_{u \geq 0} P(u) = \frac{\left[\max\left(\nu + d h(N(\pi)), 0\right)\right]^2 - \nu^2}{2d}.
\]

Applying this result separately to $u_1$ and $u_2$, we obtain:
\[
\begin{aligned}
\min_{u_1 \geq 0} \quad & \nu_1 (N(\pi) - C_1 + u_1) + \frac{d_1}{2}(N(\pi) - C_1 + u_1)^2
= \frac{\left[\max\left(\nu_1 + d_1(N(\pi) - C_1), 0\right)\right]^2 - \nu_1^2}{2d_1}, \\
\min_{u_2 \geq 0} \quad & \nu_2 (C_2 - N(\pi) + u_2) + \frac{d_2}{2}(C_2 - N(\pi) + u_2)^2
= \frac{\left[\max\left(\nu_2 + d_2(C_2 - N(\pi)), 0\right)\right]^2 - \nu_2^2}{2d_2}.
\end{aligned}
\]

Thus, minimizing over $u_1$ and $u_2$ yields the following dual function:
\[
\Omega(\nu_1, \nu_2) = \min_{\pi} f(\pi) + g_1(\nu_1) + g_2(\nu_2),
\]
where $g_1(\nu_1)$ and $g_2(\nu_2)$ are as defined in the theorem.

Finally, solving the original constrained problem is equivalent to solving the following unconstrained dual optimization:
\[
\max_{\nu_1 \geq 0, \nu_2 \geq 0} \Omega(\nu_1, \nu_2).
\]
\end{proof}

\section{Environment Descriptions}

\textbf{Sumo-Human.} A symmetric multi-agent environment where two humanoid agents engage in a physical contest within a circular arena. The objective is to remain inside the arena while attempting to force the opponent out or cause them to fall.

\textbf{Kick-and-Defend.} An asymmetric task involving two agents: one attempts to kick a ball into a goal, while the other acts as a goalkeeper attempting to block the shot. The agents operate within a confined rectangular field.

\textbf{You-Shall-Not-Pass.} A competitive environment where one agent attempts to advance forward across a designated boundary, while the opposing agent attempts to prevent passage through physical obstruction.

\textbf{Highway merge from on-ramp.}
An adversarial vehicle attempts to interfere with the ego-vehicle merging from a highway on-ramp by blocking the gap or forcing it to slow down, with the aim of causing a failed merge or unsafe maneuver.

\textbf{Highway cut-in from on-ramp.}
An adversarial vehicle cuts into the ego-vehicle’s lane from the on-ramp at a critical moment, aiming to force the ego-vehicle into sudden braking or lateral deviation, increasing the risk of failure.

\textbf{Highway exit.}
As the ego-vehicle prepares to exit the highway, an adversarial vehicle deliberately obstructs the exit path or prevents safe lane changes, attempting to make the ego-vehicle miss the off-ramp or perform unsafe actions.

\section{LLMs Prompts Example}
\subsection{Reward Generator Prompts Example(Sumo-Human)}

When responding to my questions, provide only the necessary code that directly answers them, without adding any additional text before or after the code. Your primary objective is to optimize the RL model by improving its {winning rate} and ensuring that the agent actively defeats its opponent within the competition time.

The agent must first learn to stand properly before earning rewards by performing strategic actions such as approaching the opponent, knocking the opponent down, pushing the opponent out of the ring, and avoiding self-falls or self-exits. To achieve this, design a {reward function} that enables the agent to learn a winning strategy more efficiently and stably. Specifically, implement a Python class named "RewardCalculator", in which the "compute()" function serves as the reward function responsible for calculating the reward.
 The compute() function takes the following parameters: the agent’s previous state (\texttt{s1}), current state (\texttt{s2}), current action (\texttt{a2}), victory status (\texttt{victory\_status}, which can be \texttt{"win"}, \texttt{"loss"}, or \texttt{None}), the opponent’s state after being affected by \texttt{a2} (\texttt{s\_o}), the opponent’s agent object (\texttt{agent}), and the agent’s current success rate (\texttt{rate}). The function must compute a reward based on the observed state transition and the action taken. {Do not provide any additional code except for:} 

\begin{center}
\texttt{def reward(s1, s2, a2, victory\_status, s\_o, agent, rate):}
\end{center}

Ensure that the function returns both a reward value and a dictionary containing key metrics relevant to solving the problem. The reward function must depend on all input parameters and exhibit a certain level of complexity to encourage the agent to identify and exploit the opponent's weaknesses through adversarial strategies. The {core objective of the reward function is to enhance the agent’s success rate (winning\_rate)}. The reward should incentivize the agent to actively defeat the opponent while avoiding unnecessary penalties that could hinder training.

You may invoke the opponent's policy function in the reward function using:

\begin{center}
    \texttt{agent.act\_(observation=s\_o, stochastic=False)[0]}
\end{center}

By comparing the opponent's actions in the current state and in a perturbed state, you can estimate the opponent’s sensitivity to state changes. This sensitivity should be incorporated into the reward function to incentivize the agent to trigger the opponent's instability.

The designed reward function should consist of two main components: one related to increasing the {winning rate} and the other ensuring the agent maintains {proper standing behavior}, which serves as the foundation for winning. The weighting between these two components should be determined based on the agent’s current success rate.

This environment simulates a competitive wrestling scenario in which two 3D bipedal robots engage in a match within a circular arena. Each robot consists of a torso (abdomen), a pair of arms, and a pair of legs, where each leg has three joints and each arm has two. The task is to control one of the robots by applying torques to its joints to defeat the opponent.

The {observation space} consists of 395 dimensions and is structured as follows. The first 24 dimensions (\texttt{obs[0:24]}) represent the robot’s global position and the relative positions of its joints, including the torso’s global position (3D) and the rotational positions of the abdomen, hips, knees, shoulders, and elbows. The next 23 dimensions (\texttt{obs[24:47]}) store generalized velocity information, including the linear velocity of the torso and the angular velocities of all joints. Dimensions \texttt{obs[47:177]} describe the inertial properties of each major body part, including mass, center of mass position, and moments of inertia. Relative velocity information is recorded in \texttt{obs[177:255]}, where each body part has 6 dimensions representing linear and angular velocities. The actuator torques applied to each joint, which control the robot’s movement, are stored in \texttt{obs[255:278]}. External contact forces and torques applied to major body components are found in \texttt{obs[278:356]}. The opponent’s position state, structured identically to \texttt{obs[0:24]}, is stored in \texttt{obs[356:380]}. The next two dimensions (\texttt{obs[380:382]}) encode the relative distances between the two robots along the x and y axes. The torso’s rotation matrix, which defines its orientation in 3D space, is given in \texttt{obs[382:391]}. Finally, the last four dimensions (\texttt{obs[391:395]}) represent the radius of the arena, the robot’s distance to the boundary, the opponent’s distance to the boundary, and the remaining competition time.

The \textbf{action space} consists of a 17-dimensional vector, where each element represents the torque applied to a joint, ranging from \texttt{-0.4} to \texttt{0.4} Nm. The action controls include rotational torques for the abdomen along three axes, rotations of the left and right hip joints, flexion of the left and right knee joints, as well as shoulder and elbow movements.

The match outcome is determined by the following {victory conditions}. A robot wins if the opponent either falls (\texttt{z-coordinate of the torso < 1}) or exits the ring within the competition time. If neither condition is met before the time limit, the agent is considered to have lost.

Your previous reward function is given as:

\begin{center}
    \texttt{\{code\}}
\end{center}

During the PPO training over 800 epochs, the reward component dictionaries obtained every 100 epochs are as follows.

\begin{center}
    \texttt{\{details\}}
\end{center}

where \texttt{step} represents the average number of steps per episode, \texttt{totalreward} and \texttt{groundreward} represent the average episode rewards obtained from the previously defined reward function and the default system reward, respectively, and \texttt{winning\_rate} represents the agent’s success rate, which is the primary metric for optimization.

\textbf{Guidelines for Improving the Reward Function:} Carefully analyze the feedback from the reward components during training and design an improved reward function to better solve the task and enhance the success rate. If the success rate (\texttt{winning\_rate}) remains close to zero, the reward function must be restructured to explicitly incentivize goal-directed behaviors. If certain reward components show little variation, RL may struggle to optimize them. Possible solutions include adjusting scaling or introducing a temperature parameter, redesigning components to provide more meaningful learning signals, or removing components that do not contribute to performance. It is also essential to prevent certain reward components from dominating the total reward by scaling them appropriately. Additionally, applying nonlinear transformations such as \texttt{torch.exp} or normalization techniques can smooth reward values and control their influence.

{Key considerations in reward design:} The reward function should ensure that the agent successfully wins within the competition time, balance the magnitude of different reward components to avoid excessive dominance by any single term, and include both the total reward value and a breakdown of reward components in the output. When refining the reward function, gradually adjust it based on training feedback to align it with the optimization objectives.

\subsection{Reward Evaluator Prompts Example(Sumo-Human)}
Based on the following reward functions, as well as the changes in reward components and win rates during training, identify the best-performing reward function. You are not allowed to output anything except the best reward function and its changes in reward components and win rates. 

The environment description is as follows: This environment simulates a competitive wrestling scenario in which two 3D bipedal robots engage in a match within a circular arena. Each robot consists of a torso (abdomen), a pair of arms, and a pair of legs, where each leg has three joints and each arm has two. The task is to control one of the robots by applying torques to its joints to defeat the opponent.
The {observation space} consists of 395 dimensions and is structured as follows. The first 24 dimensions (\texttt{obs[0:24]}) represent the robot’s global position and the relative positions of its joints, including the torso’s global position (3D) and the rotational positions of the abdomen, hips, knees, shoulders, and elbows. The next 23 dimensions (\texttt{obs[24:47]}) store generalized velocity information, including the linear velocity of the torso and the angular velocities of all joints. Dimensions \texttt{obs[47:177]} describe the inertial properties of each major body part, including mass, center of mass position, and moments of inertia. Relative velocity information is recorded in \texttt{obs[177:255]}, where each body part has 6 dimensions representing linear and angular velocities. The actuator torques applied to each joint, which control the robot’s movement, are stored in \texttt{obs[255:278]}. External contact forces and torques applied to major body components are found in \texttt{obs[278:356]}. The opponent’s position state, structured identically to \texttt{obs[0:24]}, is stored in \texttt{obs[356:380]}. The next two dimensions (\texttt{obs[380:382]}) encode the relative distances between the two robots along the x and y axes. The torso’s rotation matrix, which defines its orientation in 3D space, is given in \texttt{obs[382:391]}. Finally, the last four dimensions (\texttt{obs[391:395]}) represent the radius of the arena, the robot’s distance to the boundary, the opponent’s distance to the boundary, and the remaining competition time.
The \textbf{action space} consists of a 17-dimensional vector, where each element represents the torque applied to a joint, ranging from \texttt{-0.4} to \texttt{0.4} Nm. The action controls include rotational torques for the abdomen along three axes, rotations of the left and right hip joints, flexion of the left and right knee joints, as well as shoulder and elbow movements.
The match outcome is determined by the following {victory conditions}. A robot wins if the opponent either falls (\texttt{z-coordinate of the torso < 1}) or exits the ring within the competition time. If neither condition is met before the time limit, the agent is considered to have lost.

The reward functions, reward component changes, and win rates are listed below.
\begin{center}
    \texttt{\{code\}}
\end{center}

\begin{center}
    \texttt{\{details\}}
\end{center}
\section{Adversarial Rewards}
 We provide the optimal adversarial reward function for each environment.
\subsection{Sumo-Human}
\begin{verbatim}
class RewardCalculator:
    def __init__(self,
                 gamma=0.99,
                 win_bonus=8.0,
                 loss_penalty=-4.0,
                 energy_coef=0.005,
                 step_penalty=-0.001):
        self.gamma        = float(gamma)
        self.win_bonus    = float(win_bonus)
        self.loss_penalty = float(loss_penalty)
        self.energy_coef  = float(energy_coef)
        self.step_pen     = float(step_penalty)
        self._stats       = {"wins": 0, "games": 0, "steps": 0}
        self._ep_steps    = 0

    # --------------------------------------------------------------
    def compute(self, s1, s2, a2, victory_status,
                s_oppo, reward_adv, reward_opp):
        import numpy as np

        # -------- stats / episode step --------
        self._update_stats(victory_status)
        self._ep_steps += 1

        s1, s2, s_oppo, a2 = map(lambda x: np.asarray(x, np.float32),
                                 (s1, s2, s_oppo, a2))

        # -------- adaptive weights  --------
        win_rate   = self._stats["wins"] / max(self._stats["games"], 1)
        combat_w   = 0.3 + 0.7 * win_rate
        stand_w    = 1.0 - combat_w

        # -------- potential helpers --------
        def stand_phi(s):
            height = np.clip(s[2] - 1.0, -1.0, 1.0)
            upright = np.clip(s[390], 0.0, 1.0)
            vel_pen = -0.5 * np.tanh(np.linalg.norm(s[24:27]) / 2.0)
            return 0.6 * height + 0.4 * upright + vel_pen        # (-1,1)

        def combat_phi(s_self, s_enemy):
            arena_r   = s_self[391] if s_self.size > 391 else 3.0
            dist      = np.linalg.norm(s_self[380:382]) / max(arena_r, 1e-6)
            ring_adv  = np.tanh((s_enemy[393] - s_self[392]) /
                                (arena_r * 0.5 + 1e-6))
            tilt      = np.clip(s_enemy[390] - s_self[390], 0.0, 1.0)
            push      = np.clip(1.0 - s_enemy[2], 0.0, 1.0)
            return -0.6 * dist + 0.2 * ring_adv + 0.1 * tilt + 0.1 * push

        phi1 = stand_w * stand_phi(s1) + combat_w * combat_phi(s1, s_oppo)
        phi2 = stand_w * stand_phi(s2) + combat_w * combat_phi(s2, s_oppo)
        dense_r = self.gamma * phi2 - phi1                               

        # -------- sparse & event rewards --------
        sparse_r = 0.5 * reward_adv - 0.1 * reward_opp
        terminal_bonus = (self.win_bonus if victory_status == "win"
                          else self.loss_penalty if victory_status == "loss"
                          else 0.0)

        # -------- penalties --------
        energy_pen = -self.energy_coef * float(np.sum(a2 ** 2))
        step_pen   = self.step_pen

        # -------- total --------
        total = dense_r + sparse_r + terminal_bonus + energy_pen + step_pen

        info = dict(
            total_reward   = float(total),
            dense_reward   = float(dense_r),
            stand_weight   = float(stand_w),
            combat_weight  = float(combat_w),
            sparse_reward  = float(sparse_r),
            terminal_bonus = float(terminal_bonus),
            energy_penalty = float(energy_pen),
            step_penalty   = float(step_pen),
            win_rate       = float(win_rate),
            steps_global   = int(self._stats["steps"]),
            steps_episode  = int(self._ep_steps),
        )
        return float(total), info

    # --------------------------------------------------------------
    def _update_stats(self, victory_status: str):
        if victory_status in ("win", "loss"):
            self._stats["games"] += 1
            self._stats["wins"] += (victory_status == "win")
            self._ep_steps = 0            # reset per episode
        self._stats["steps"] += 1
\end{verbatim}
\subsection{You-Shall-Not-Pass}
\begin{verbatim}
    class RewardCalculator:
    def __init__(
        self,
        gamma: float = 0.995,
        dense_scale: float = 12.0,
        win_bonus: float = 15.0,
        loss_penalty: float = -6.0,
        energy_coef: float = 0.0008,
        step_penalty: float = -0.0005,
        ema_alpha: float = 0.1,
    ):
        self.gamma = float(gamma)
        self.dense_scale = float(dense_scale)
        self.win_bonus = float(win_bonus)
        self.loss_penalty = float(loss_penalty)
        self.energy_coef = float(energy_coef)
        self.step_penalty = float(step_penalty)
        self.ema_alpha = float(ema_alpha)

        self._stats = {"wins": 0, "games": 0, "steps": 0}
        self._ep_steps = 0
        self._ema_win_rate = 0.0

    # --------------------------------------------------------------
    def compute(
        self,
        s1,
        s2,
        a2,
        victory_status,
        s_o,
        reward_adv,
        reward_opp,
    ):
        s1, s2, s_o, a2 = map(lambda x: np.asarray(x, np.float32), (s1, s2, s_o, a2))
        self._update_stats(victory_status)
        self._ep_steps += 1

        # ----------- adaptive weights ----------------
        stand_w = np.clip(0.7 * (1.0 - self._ema_win_rate), 0.3, 0.8)
        block_w = 1.0 - stand_w

        # ----------- potentials ----------------------
        def _stand_phi(s):
            height = np.clip((s[2] - 0.3) / 0.7, 0.0, 1.0)
            vel_pen = -0.3 * np.tanh(np.linalg.norm(s[24:27]))
            return 0.8 * height + vel_pen

        def _block_phi(b, w):
            ahead = np.tanh((w[0] - b[0]) * 2.0)         
            lateral = -np.tanh(np.abs(b[1] - w[1]) * 1.5)
            prog_pen = -np.tanh(-w[0])                  
            return 0.5 * ahead + 0.3 * lateral + 0.2 * prog_pen

        phi1 = stand_w * _stand_phi(s1) + block_w * _block_phi(s1, s_o)
        phi2 = stand_w * _stand_phi(s2) + block_w * _block_phi(s2, s_o)
        dense_r = self.dense_scale * (self.gamma * phi2 - phi1)

        # ----------- sparse & terminal ----------------
        sparse_r = 0.3 * reward_adv - 0.1 * reward_opp
        terminal_r = (
            self.win_bonus if victory_status == "win"
            else self.loss_penalty if victory_status == "loss"
            else 0.0
        )

        # ----------- penalties -----------------------
        energy_pen = -self.energy_coef * float(np.sum(a2 ** 2))
        step_pen = self.step_penalty

        # ----------- total reward --------------------
        total = dense_r + sparse_r + terminal_r + energy_pen

        info = dict(
            total_reward=float(total),
            dense_reward=float(dense_r),
            sparse_reward=float(sparse_r),
            terminal_reward=float(terminal_r),
            energy_penalty=float(energy_pen),
            step_penalty=float(step_pen),
            stand_weight=float(stand_w),
            block_weight=float(block_w),
            ema_win_rate=float(self._ema_win_rate),
            steps_global=int(self._stats["steps"]),
            steps_episode=int(self._ep_steps),
        )
        return float(total), info

    # --------------------------------------------------------------
    def _update_stats(self, victory_status: str):
        if victory_status in ("win", "loss"):
            self._stats["games"] += 1
            self._stats["wins"] += (victory_status == "win")
            current_win_rate = self._stats["wins"] / self._stats["games"]
            self._ema_win_rate = (
                self.ema_alpha * current_win_rate
                + (1.0 - self.ema_alpha) * self._ema_win_rate
            )
            self._ep_steps = 0
        self._stats["steps"] += 1
\end{verbatim}
\subsection{Kick-and-Defend}
\begin{verbatim}
    class RewardCalculator:
    def __init__(
        self,
        win_bonus: float = 2.0,
        loss_penalty: float = -2.0,
        dist_weight: float = 0.02,
        delta_weight: float = 0.03,
        align_weight: float = 0.015,
        threat_weight: float = 0.1,
        threat_radius: float = 2.0,
        gate_penalty_once: float = -0.1,
        stance_penalty_once: float = -0.5,
        energy_weight: float = -0.001,
        adv_diff_weight: float = 0.4,
        time_penalty: float = -0.001,
        clip_limit: float = 115.0,
        align_sigma: float = 0.5,
    ):
        self.W   = win_bonus
        self.L   = loss_penalty
        self.wd  = dist_weight
        self.w_delta = delta_weight
        self.wa  = align_weight
        self.wt  = threat_weight
        self.rt  = threat_radius
        self.wg  = gate_penalty_once
        self.ws  = stance_penalty_once
        self.we  = energy_weight
        self.wr  = adv_diff_weight
        self.wp  = time_penalty
        self.M   = clip_limit
        self.sig = align_sigma

    def compute(self, s1, s2, a2, victory_status, s_o, reward_adv, reward_opp):
        s1 = np.asarray(s1, dtype=np.float32)
        s2 = np.asarray(s2, dtype=np.float32)
        a2 = np.asarray(a2, dtype=np.float32)
        
        if victory_status == "win":
            terminal = self.W
        elif victory_status == "loss":
            terminal = self.L
        else:
            terminal = 0.0

        dx1 = max(0.0, float(s1[381]))
        dx2 = max(0.0, float(s2[381]))
        r_dist  = self.wd * np.tanh(dx2 / 20.0)
        r_delta = self.w_delta * np.tanh((dx2 - dx1) / 4.0)

        dy = float(s2[379])
        r_align = self.wa * np.exp(-dy * dy / (2 * self.sig * self.sig))

        left_gap = float(s2[382])
        right_gap = float(s2[383])
        between = (left_gap * right_gap) < 0.0
        if between and dx2 <= self.rt:
            r_threat = self.wt * (self.rt - dx2) / self.rt
            r_gate   = self.wg
        else:
            r_threat = 0.0
            r_gate   = 0.0

        z1 = float(s1[0])
        z2 = float(s2[0])
        r_stance = self.ws if (z1 >= 0.75 and z2 < 0.75) else 0.0

        r_energy  = self.we * np.sum(a2 * a2)
        r_advdiff = self.wr * (reward_adv - reward_opp)

        total = (
            terminal + r_dist + r_delta + r_align +
            r_threat + r_gate + r_stance +
            r_energy + r_advdiff
        )
        total = float(np.clip(total, -self.M, self.M))

        info = {
            "total_reward": total,
            "terminal": terminal,
            "r_dist": r_dist,
            "r_delta": r_delta,
            "r_align": r_align,
            "r_threat": r_threat,
            "r_gate": r_gate,
            "r_stance": r_stance,
            "r_energy": r_energy,
            "r_advdiff": r_advdiff,
        }
        return total, info
\end{verbatim}

\section{Training and Optimization Settings}

The training process consists of three main stages: reward optimization, adversarial policy pre-training, and fine-tuning with critical state identification.

\textbf{Reward Optimization.}  
We use GPT-4o as both the Reward Generator and Evaluator. In each of four optimization rounds, four candidate reward functions are generated and evaluated, resulting in 32 total API calls. The best-performing reward is selected and used for subsequent attacker training.

\textbf{Adversarial Policy Training.}
The attacker policy is trained using PPO with the selected adversarial reward. PPO is configured with a learning rate of $3 \times 10^{-4}$, a clip range of $0.2$, 2048 rollout steps per update, 4 optimization epochs, and a batch size of 512. The training is performed in parallel with 8 environments to accelerate data collection and improve sample efficiency.

\textbf{Critical State Identification.}  
We adopt the constrained optimization formulation described in Theorem~2, where the number of perturbed states is bounded between $C_2 = 20$ and $C_1 = 40$. Penalty coefficients are set to $d_1 = d_2 = 5$. The critical state selection policy is optimized using PPO with the same hyperparameters as the attacker.

\textbf{Fine-tuning.}  
During fine-tuning, the attacker receives a composite reward defined as $R_{\text{total}} = R + \lambda R_{ft}$, where $R$ is the original adversarial reward and $R_{ft}$ is the deviation-based fine-tuning signal. We set $\lambda = 0.3$ to balance the two components. The critical state selector is re-optimized every $K = 10$ attacker updates to reflect the evolving attack policy.

\section{Broader Impacts}
This work investigates adversarial attacks on reinforcement learning systems in black-box settings. On the positive side, our proposed ARCS framework can serve as a tool for stress-testing RL agents deployed in safety-critical applications, such as autonomous driving or industrial control, helping researchers identify and mitigate potential vulnerabilities. This contributes to the broader goal of developing more robust and trustworthy AI systems.

However, we also acknowledge the potential for misuse. The techniques proposed—particularly the automated construction of adversarial reward functions and identification of critical decision points—could be exploited to disrupt real-world RL systems if applied maliciously. While our experiments are entirely conducted in simulation and for research purposes only, we recognize the importance of responsible use and encourage further discussion on safeguards and ethical deployment.

We believe that the benefits of advancing robustness research outweigh the risks, but caution must be exercised in any real-world application of adversarial techniques.

\end{document}